\documentclass[11pt,a4paper]{article}

\usepackage[margin=25mm]{geometry}
\usepackage{amsmath,amssymb,amsthm}
\usepackage{mathtools}
\usepackage{booktabs,tabularx}
\usepackage{enumitem}
\setlist[itemize]{label=$\bullet$}
\usepackage{float}
\usepackage{tikz}
\usetikzlibrary{arrows.meta,positioning}
\usepackage[hidelinks]{hyperref}
\usepackage[capitalise,nameinlink]{cleveref}

\newtheorem{theorem}{Theorem}[section]
\newtheorem{proposition}[theorem]{Proposition}

\newtheorem{corollary}[theorem]{Corollary}
\theoremstyle{definition}
\newtheorem{definition}[theorem]{Definition}
\newtheorem{example}[theorem]{Example}
\theoremstyle{remark}
\newtheorem{remark}[theorem]{Remark}

\crefname{correction}{Correction}{Corrections}

\newcommand{\Path}{\mathbf{Path}}
\newcommand{\Emb}{\mathbf{Emb}}
\newcommand{\cod}{\operatorname{cod}}
\newcommand{\dom}{\operatorname{dom}}
\newcommand{\lab}{\lambda}
\newcommand{\Phiweak}{\Phi^{\simeq}}

\title{Categorical Approach to Conflict Resolution:\\
A Corrected Correspondence between Category Theory\\
and the Graph Model for Conflict Resolution}

\author{Yukiko Kato\\[2pt]
\small Lynx Technologies Inc., Tokyo, Japan\\
\small Institute of Science Tokyo, Tokyo, Japan\\
\small \texttt{ykato@lynxjp.com}}

\date{Version 2}

\begin{document}
\maketitle

\begin{abstract}
This note is a substantially revised version of the author's earlier preprint
(version~1, 2023), which proposed a ``Categorical Graph Model for Conflict
Resolution'' (C-GMCR). Several of the central definitions and claims of
version~1 are incorrect, and this version replaces them. We first observe that
the states and one-step moves of a graph model do not form a category; the
correct construction is the free category on a quiver whose arrows are moves
labelled by the decision maker (DM) who controls them. Controller labels are
functorial: they define a functor into the free monoid on the set of DMs, and
the legal move sequences underlying coalition reachability are exactly the
paths whose label words contain no immediate repetition. Preferences, in
contrast, are not functorial. We show that a map representing a DM's
preference in a preordered set exists only if the preference is transitive,
and that even then it extends to a functor on the path category if and only if
every move, regardless of its controller, is weakly improving for the focal
DM; in that case general metarationality, symmetric metarationality and
sequential stability all collapse to Nash stability for the focal DM. Preferences therefore enter the
categorical picture as a selection of a subquiver of improvements, not as a
functor. We restate the four standard stability concepts in path language,
work them out on the Prisoner's Dilemma, and take a first step towards
comparing graph models via induced embeddings: Nash stability is reflected by
induced embeddings, whereas general metarationality, symmetric metarationality
and sequential stability are neither preserved nor reflected. A section lists each correction to version~1
explicitly.
\end{abstract}

\paragraph{Note on this version.}
Version~1 of this preprint was accepted to IEEE SMC~2023 but was neither
presented nor published. This version supersedes it. \Cref{sec:corrections}
lists the definitions and propositions of version~1 that do not hold and
states what replaces them.

\section{Introduction}

The Graph Model for Conflict Resolution (GMCR) \cite{KHF1987,FHK1993}
represents a conflict by a set of states, a directed graph of unilateral moves
for each decision maker (DM), and a preference relation for each DM. Because a
graph model is built from states and arrows between them, it is natural to ask
how it relates to category theory \cite{EM1945,MacLane1998,Awodey2010,Riehl2016},
whose basic data are also objects and arrows. Version~1 of this preprint
proposed such a correspondence under the name C-GMCR. The motivation remains
sound, but the correspondence proposed there was not correct: it treated the
set of one-step moves as if it were closed under composition, it represented
preferences by a functor that exists only in degenerate cases, and it restated
the standard stability concepts in forms that do not coincide with the
originals.

This version gives a corrected correspondence. Its guiding observation is that
the two ingredients of a graph model behave very differently with respect to
composition of moves:
\begin{itemize}[leftmargin=*]
  \item \emph{who moves} is compositional: controller labels define a functor
        from the path category of the model to the free monoid on the set of
        DMs (\cref{prop:labelling}), and the legal sequences used in coalition
        reachability are characterised by a condition on label words
        (\cref{prop:legal});
  \item \emph{how the result is valued} is not compositional: a preference
        functor exists only when every move is weakly improving for the DM in
        question (\cref{thm:nofunctor}), and in that case all non-Nash
        stability concepts collapse for that DM (\cref{cor:collapse}).
\end{itemize}
Stability concepts compare the endpoint of a sequence of moves with its
starting state. This comparison is exactly the part of GMCR that is not
captured by composition alone, which is why GMCR solution concepts do not
arise from the category structure by itself.

The corrected categorical formulation is therefore not intended to replace
the GMCR stability machinery, nor does it produce new solution concepts. Its
role is structural. It separates the compositional move dynamics from the
non-compositional preference comparisons, it gives a precise representation
of the legal move sequences underlying coalition reachability, and it supplies
a language for studying maps between graph models.

The note is organised as follows. \Cref{sec:prelim} recalls the graph model
and the standard stability concepts. \Cref{sec:quiver} constructs the labelled
quiver and its free category and characterises coalition reachability.
\Cref{sec:pref} proves the non-functoriality results for preferences.
\Cref{sec:stability} restates the stability concepts in path language and
works them out on the Prisoner's Dilemma. \Cref{sec:emb} introduces induced
embeddings between graph models and records their first preservation
properties. \Cref{sec:corrections} lists the corrections to version~1.

\section{Graph models and standard stability concepts}\label{sec:prelim}

\begin{definition}[Graph model]\label{def:gm}
A \emph{graph model} is a tuple $G=(N,S,(A_i)_{i\in N},(\succsim_i)_{i\in N})$,
where $N$ is a finite set of DMs with $|N|\ge 2$, $S$ is a finite set of
states with $|S|\ge 2$, $A_i\subseteq S\times S$ is the set of moves
controlled by DM~$i$, and $\succsim_i$ is a complete binary relation on $S$
(for all $s,t\in S$, $s\succsim_i t$ or $t\succsim_i s$). Completeness implies
reflexivity. Transitivity is \emph{not} assumed. We write $s\succ_i t$ if
$s\succsim_i t$ and not $t\succsim_i s$, and $s\sim_i t$ if both
$s\succsim_i t$ and $t\succsim_i s$.
\end{definition}

A pair $(s,s)\in A_i$ may be used to record that DM~$i$ can remain at $s$.
Since $s\succ_i s$ never holds, such a pair is never an improvement. To keep
coalition reachability in its standard form, reachable lists are defined over
proper moves only:
\[
  R_i(s)=\{t\in S : (s,t)\in A_i,\ t\neq s\},\qquad
  R_i^{+}(s)=\{t\in R_i(s) : t\succ_i s\}.
\]
We also write $\Phiweak_i(s)=\{t\in S : s\succsim_i t\}$ for the set of
states that DM~$i$ does not strictly prefer to $s$.

\begin{definition}[Coalition reachability {\cite{FHK1993}}]\label{def:RH}
Let $H\subseteq N$ with $H\neq\emptyset$ and $s\in S$. The set $R_H(s)$ and
the sets $\Omega_H(s,t)\subseteq H$ of \emph{last movers} are the least sets
satisfying:
\begin{enumerate}[label=(\roman*)]
  \item if $j\in H$ and $t\in R_j(s)$, then $t\in R_H(s)$ and
        $j\in\Omega_H(s,t)$;
  \item if $t'\in R_H(s)$, $j\in H$, $t\in R_j(t')$ and
        $\Omega_H(s,t')\neq\{j\}$, then $t\in R_H(s)$ and
        $j\in\Omega_H(s,t)$.
\end{enumerate}
The sets $R_H^{+}(s)$ and $\Omega_H^{+}(s,t)$ are defined in the same way with
$R_j$ replaced by $R_j^{+}$.
\end{definition}

Condition (ii) forbids the same DM from moving twice in succession. For a
singleton $H=\{j\}$, it gives $R_{\{j\}}(s)=R_j(s)$.

\begin{definition}[Standard stability concepts
{\cite{Nash1950,Nash1951,Howard1971,FraserHipel1979,FraserHipel1984,FHK1993}}]
\label{def:stab}
Let $s\in S$, $i\in N$, and write $N{-}i$ for $N\setminus\{i\}$. State $s$ is
\begin{enumerate}[label=(\alph*)]
  \item \emph{Nash stable} for $i$ if $R_i^{+}(s)=\emptyset$;
  \item \emph{general metarational} (GMR) for $i$ if for every
        $s_1\in R_i^{+}(s)$, $R_{N-i}(s_1)\cap\Phiweak_i(s)\neq\emptyset$;
  \item \emph{symmetric metarational} (SMR) for $i$ if for every
        $s_1\in R_i^{+}(s)$ there is $s_2\in R_{N-i}(s_1)\cap\Phiweak_i(s)$
        such that $s\succsim_i s_3$ for every $s_3\in R_i(s_2)$;
  \item \emph{sequentially stable} (SEQ) for $i$ if for every
        $s_1\in R_i^{+}(s)$, $R_{N-i}^{+}(s_1)\cap\Phiweak_i(s)\neq\emptyset$.
\end{enumerate}
A state is an \emph{equilibrium} under a concept if it is stable for every
$i\in N$ under that concept.
\end{definition}

\begin{proposition}\label{prop:implications}
For every $i\in N$ and $s\in S$: Nash $\Rightarrow$ SMR $\Rightarrow$ GMR and
Nash $\Rightarrow$ SEQ $\Rightarrow$ GMR.
\end{proposition}
\begin{proof}
If $R_i^+(s)=\emptyset$, conditions (b)--(d) hold vacuously. SMR implies GMR
because (c) exhibits an element of $R_{N-i}(s_1)\cap\Phiweak_i(s)$. SEQ
implies GMR because $R_{N-i}^{+}(s_1)\subseteq R_{N-i}(s_1)$, which follows by
induction on \cref{def:RH} from $R_j^{+}(t)\subseteq R_j(t)$.
\end{proof}

\section{The labelled quiver of a graph model and its free category}
\label{sec:quiver}

A \emph{quiver} $Q$ consists of a set $Q_0$ of vertices, a set $Q_1$ of
arrows, and maps $\dom,\cod\colon Q_1\to Q_0$. Its \emph{free category}
$\Path(Q)$ has the vertices of $Q$ as objects and the finite paths of $Q$ as
morphisms, including an empty path $\mathrm{id}_s$ at each vertex $s$;
composition is concatenation \cite[Sec.~II.7]{MacLane1998}. The free category
has the following universal property: for every category $\mathcal{C}$,
functors $\Path(Q)\to\mathcal{C}$ correspond bijectively to quiver morphisms
from $Q$ to the underlying quiver of $\mathcal{C}$.

\begin{definition}[Labelled quiver of a graph model]\label{def:quiver}
The \emph{labelled quiver} $Q_G$ of a graph model $G$ has vertex set $S$ and
one arrow $a_{i,s,t}\colon s\to t$ for each $i\in N$ and each
$(s,t)\in A_i$ with $s\neq t$. The \emph{label} of this arrow is
$\lab(a_{i,s,t})=i$. The \emph{improvement subquiver} $Q_G^{+}\subseteq Q_G$
consists of the arrows $a_{i,s,t}$ with $t\succ_i s$.
\end{definition}

If two DMs both control a move from $s$ to $t$, $Q_G$ has two parallel arrows
from $s$ to $t$ with different labels. The two arrows must be kept distinct:
which DM moves determines which preference is relevant and which coalition
sequences are legal.

\begin{remark}\label{rem:notcategory}
The set of one-step moves $\bigcup_i A_i$ is generally not closed under
composition: if $(s,t)\in A_i$ and $(t,u)\in A_j$, there need not be any DM
controlling a move from $s$ to $u$. Therefore the states together with the
one-step moves do not form a category. The composite of two moves exists in
$\Path(Q_G)$ as a path of length two, not as a move.
\end{remark}

\begin{proposition}[Controller labels are functorial]\label{prop:labelling}
Let $N^{*}$ be the free monoid on $N$, regarded as a category with one object.
The assignment $a\mapsto\lab(a)$ extends uniquely to a functor
$\lab\colon\Path(Q_G)\to N^{*}$ sending a path $a_1\cdots a_k$ (traversed in
the order $a_1,\ldots,a_k$) to the word $\lab(a_1)\cdots\lab(a_k)$ and each
identity to the empty word.
\end{proposition}
\begin{proof}
A quiver morphism from $Q_G$ to the underlying quiver of the one-object
category $N^{*}$ is exactly an assignment of an element of $N^{*}$ to each
arrow. The claim follows from the universal property of $\Path(Q_G)$.
\end{proof}

For $H\subseteq N$ nonempty, call a path $p$ \emph{$H$-legal} if it is
nonempty, every letter of $\lab(p)$ lies in $H$, and no two consecutive
letters of $\lab(p)$ are equal.

\begin{proposition}[Legal paths characterise coalition reachability]
\label{prop:legal}
For every $s,t\in S$ and $j\in H$, the following are equivalent:
\begin{enumerate}[label=(\arabic*)]
  \item $t\in R_H(s)$ and $j\in\Omega_H(s,t)$;
  \item there is an $H$-legal path $p\colon s\to t$ in $Q_G$ whose last
        letter is $j$.
\end{enumerate}
In particular, $R_H(s)=\{\cod(p): p \text{ an $H$-legal path in $Q_G$ with }
\dom(p)=s\}$. The same statements hold for $R_H^{+}$, $\Omega_H^{+}$ and
$Q_G^{+}$.
\end{proposition}
\begin{proof}
(2)$\Rightarrow$(1), by induction on the length of $p$. If $p$ has length
one, it is an arrow $a_{j,s,t}$, so $t\in R_j(s)$ and (i) of \cref{def:RH}
gives (1). If $p$ has length $k+1\ge 2$, write $p$ as an $H$-legal path
$q\colon s\to t'$ followed by an arrow $a_{j,t',t}$. The last letter $j'$ of
$q$ differs from $j$ by legality. By the induction hypothesis,
$t'\in R_H(s)$ and $j'\in\Omega_H(s,t')$, so $\Omega_H(s,t')\neq\{j\}$. Since
$t\in R_j(t')$, rule (ii) gives (1).

(1)$\Rightarrow$(2), by induction on the generation of $R_H(s)$ and
$\Omega_H$ as least sets. If (1) is obtained by rule (i), the arrow
$a_{j,s,t}$ is the required path. If it is obtained by rule (ii) from
$t'\in R_H(s)$ with $\Omega_H(s,t')\neq\{j\}$ and $t\in R_j(t')$, note that
$\Omega_H(s,t')$ is nonempty, since every application of a rule that puts
$t'$ into $R_H(s)$ also puts a DM into $\Omega_H(s,t')$. Hence there is
$j'\in\Omega_H(s,t')$ with $j'\neq j$. By the induction hypothesis there is
an $H$-legal path $q\colon s\to t'$ with last letter $j'$, and appending
$a_{j,t',t}$ gives an $H$-legal path to $t$ with last letter $j$.

The argument for $R_H^{+}$ is identical, with arrows taken from $Q_G^{+}$.
\end{proof}

Thus coalition reachability is determined by the path category together with
the labelling functor: it is the image under $\cod$ of the paths from $s$
whose label words are nonempty words over $H$ and avoid factors of the form
$jj$.

\section{Preferences are not functorial}\label{sec:pref}

Version~1 represented the preference of DM~$i$ by a ``preference functor''
$P_i$ from the category of states and transitions to an ordered set of
preference levels, with $P_i(s)\le P_i(t)$ if and only if $t\succsim_i s$,
and with each transition $f\colon s\to t$ sent to an arrow
$P_i(s)\to P_i(t)$. A preordered set $(L,\le)$ is regarded as a category
with a unique arrow $x\to y$ when $x\le y$ and no arrow otherwise. We examine
the object part and the arrow part of this proposal separately.

\begin{proposition}[Representation requires transitivity]
\label{prop:transitive}
For a DM $i$, the following are equivalent:
\begin{enumerate}[label=(\arabic*)]
  \item there exist a preordered set $(L,\le)$ and a map $P\colon S\to L$
        with $P(s)\le P(t)\iff t\succsim_i s$ for all $s,t\in S$;
  \item $\succsim_i$ is transitive.
\end{enumerate}
\end{proposition}
\begin{proof}
(1)$\Rightarrow$(2): if $t\succsim_i s$ and $u\succsim_i t$, then
$P(s)\le P(t)\le P(u)$, so $P(s)\le P(u)$ and $u\succsim_i s$.
(2)$\Rightarrow$(1): the relation $\precsim_i$ (the converse of $\succsim_i$)
is reflexive by completeness and transitive by assumption, so $(S,\precsim_i)$
is a preordered set and the identity map works.
\end{proof}

Since GMCR does not require transitive preferences, the object part of a
preference functor is already unavailable in general.

\begin{theorem}[When a preference functor exists]\label{thm:nofunctor}
Let $i\in N$, and let $P\colon S\to L$ be as in \cref{prop:transitive}(1).
There is a functor $F\colon\Path(Q_G)\to(L,\le)$ with $F(s)=P(s)$ for all
$s\in S$ if and only if
\begin{equation}\label{eq:allweak}
  t\succsim_i s \quad\text{for every arrow } a\colon s\to t \text{ of } Q_G,
  \text{ whatever its label.}
\end{equation}
When it exists, $F$ is unique.
\end{theorem}
\begin{proof}
By the universal property of $\Path(Q_G)$, such a functor corresponds to a
quiver morphism from $Q_G$ to the underlying quiver of $(L,\le)$ that agrees
with $P$ on vertices. Such a morphism must send each arrow $a\colon s\to t$ to
an arrow $P(s)\to P(t)$ of $L$, which exists if and only if
$P(s)\le P(t)$, that is, if and only if $t\succsim_i s$. Arrows of $L$
between given objects are unique, so the quiver morphism, and hence $F$, is
unique when it exists.
\end{proof}

Condition~\eqref{eq:allweak} states that no move by any DM, including the
opponents of $i$, ever makes DM~$i$ worse off. This excludes sanctions, which
are the reason GMCR distinguishes several stability concepts.

\begin{corollary}[Collapse to Nash stability]\label{cor:collapse}
Suppose $\succsim_i$ is transitive and \eqref{eq:allweak} holds for DM~$i$.
Then, for every $s\in S$, $s$ is GMR for $i$ if and only if $s$ is Nash
stable for $i$. Consequently, SMR and SEQ for $i$ also coincide with Nash
stability for $i$.
\end{corollary}
\begin{proof}
Suppose $s$ is not Nash stable for $i$, and let $s_1\in R_i^{+}(s)$. By
\cref{prop:legal}, every $t\in R_{N-i}(s_1)$ is the endpoint of a path from
$s_1$ each of whose arrows is weakly improving for $i$, so $t\succsim_i s_1$
by transitivity. Since $s_1\succ_i s$, transitivity gives $t\succsim_i s$; if
also $s\succsim_i t$, then $s\succsim_i s_1$, contradicting $s_1\succ_i s$.
Hence $t\succ_i s$, so $t\notin\Phiweak_i(s)$. Therefore
$R_{N-i}(s_1)\cap\Phiweak_i(s)=\emptyset$ and $s$ is not GMR for $i$. The
converse is \cref{prop:implications}, and the statement for SMR and SEQ
follows from \cref{prop:implications} as well.
\end{proof}

\begin{remark}\label{rem:selection}
Preferences therefore enter the categorical picture not as a functor but as a
\emph{selection}: $\succsim_i$ selects the arrows labelled $i$ that belong to
the improvement subquiver $Q_G^{+}$ (\cref{def:quiver}), and it is used a
second time, in the stability concepts, to compare the endpoint of a path with
its starting state. A graph model is thus best viewed as a multi-relational
structure on $S$, carrying the move relations $A_i$ and the preference
relations $\succsim_i$, with $\Path(Q_G)$ and $\lab$ organising the move
structure.
\end{remark}

\section{Stability concepts in path language}\label{sec:stability}

By \cref{prop:legal}, \cref{def:stab} can be restated in terms of paths. Write
$a\colon s\xrightarrow{\,i\,}s_1$ for an arrow of label $i$.

\begin{proposition}\label{prop:pathstab}
Let $s\in S$ and $i\in N$. Then $s$ is
\begin{enumerate}[label=(\alph*)]
  \item Nash stable for $i$ iff $Q_G^{+}$ has no arrow
        $s\xrightarrow{\,i\,}s_1$;
  \item GMR for $i$ iff for every arrow $s\xrightarrow{\,i\,}s_1$ of
        $Q_G^{+}$ there is an $(N{-}i)$-legal path $p$ in $Q_G$ with
        $\dom(p)=s_1$ and $s\succsim_i\cod(p)$;
  \item SMR for $i$ iff for every arrow $s\xrightarrow{\,i\,}s_1$ of $Q_G^{+}$
        there is an $(N{-}i)$-legal path $p$ in $Q_G$ with $\dom(p)=s_1$,
        $s\succsim_i\cod(p)$, and $s\succsim_i s_3$ for every arrow
        $\cod(p)\xrightarrow{\,i\,}s_3$ of $Q_G$;
  \item SEQ for $i$ iff for every arrow $s\xrightarrow{\,i\,}s_1$ of $Q_G^{+}$
        there is an $(N{-}i)$-legal path $p$ in $Q_G^{+}$ with $\dom(p)=s_1$
        and $s\succsim_i\cod(p)$.
\end{enumerate}
\end{proposition}
\begin{proof}
Immediate from \cref{def:stab} and \cref{prop:legal}.
\end{proof}

In each clause the final comparison involves the starting state $s$, not only
the arrows of the path. This is why the stability concepts are not obtained
from composition alone.

\begin{example}[Prisoner's Dilemma]\label{ex:pd}
Let $N=\{A,B\}$ and let the states be $s_1=(C,C)$, $s_2=(C,D)$,
$s_3=(D,C)$, $s_4=(D,D)$, where the first coordinate is A's choice. Both
options are reversible: A controls $s_1\leftrightarrow s_3$ and
$s_2\leftrightarrow s_4$, and B controls $s_1\leftrightarrow s_2$ and
$s_3\leftrightarrow s_4$. Preferences are
$s_3\succ_A s_1\succ_A s_4\succ_A s_2$ and
$s_2\succ_B s_1\succ_B s_4\succ_B s_3$. The labelled quiver $Q_G$ is shown
in \cref{fig:pd}.

\begin{figure}[H]
\centering
\begin{tikzpicture}[
  st/.style={draw,rounded corners,inner sep=4pt,minimum width=18mm},
  mv/.style={-{Stealth[length=2.2mm]},thick},
  lb/.style={font=\small}]
\node[st] (s1) at (0,2.4) {$s_1=(C,C)$};
\node[st] (s2) at (5,2.4) {$s_2=(C,D)$};
\node[st] (s3) at (0,0)   {$s_3=(D,C)$};
\node[st] (s4) at (5,0)   {$s_4=(D,D)$};
\draw[mv] (s1.10)  -- node[lb,above]{$B$} (s2.170);
\draw[mv] (s2.190) -- node[lb,below]{$B$} (s1.350);
\draw[mv] (s3.10)  -- node[lb,above]{$B$} (s4.170);
\draw[mv] (s4.190) -- node[lb,below]{$B$} (s3.350);
\draw[mv] ([xshift=-2.5mm]s1.south) -- node[lb,left]{$A$}  ([xshift=-2.5mm]s3.north);
\draw[mv] ([xshift=2.5mm]s3.north)  -- node[lb,right]{$A$} ([xshift=2.5mm]s1.south);
\draw[mv] ([xshift=-2.5mm]s2.south) -- node[lb,left]{$A$}  ([xshift=-2.5mm]s4.north);
\draw[mv] ([xshift=2.5mm]s4.north)  -- node[lb,right]{$A$} ([xshift=2.5mm]s2.south);
\end{tikzpicture}
\caption{The labelled quiver $Q_G$ of the Prisoner's Dilemma. Version~1
drew only one direction of each option.}
\label{fig:pd}
\end{figure}

The improvement subquiver $Q_G^{+}$ consists of the arrows
$s_1\xrightarrow{A}s_3$, $s_2\xrightarrow{A}s_4$, $s_1\xrightarrow{B}s_2$
and $s_3\xrightarrow{B}s_4$.
\begin{itemize}[leftmargin=*]
  \item $s_4$ has no outgoing arrow of $Q_G^{+}$, so it is a Nash
        equilibrium and hence an equilibrium under all four concepts.
  \item At $s_1$, A's improvement $s_1\xrightarrow{A}s_3$ is answered by the
        $\{B\}$-legal path $s_3\xrightarrow{B}s_4$, which lies in $Q_G^{+}$,
        with $s_1\succ_A s_4$. A's only move from $s_4$ goes to $s_2$, and
        $s_1\succ_A s_2$. Hence $s_1$ is GMR, SMR and SEQ for A, and by
        symmetry for B. It is not Nash stable.
  \item At $s_2$, A's improvement $s_2\xrightarrow{A}s_4$ can only be
        answered by $s_4\xrightarrow{B}s_3$, and $s_3\succ_A s_2$. Hence
        $s_2$ is not GMR for A, and so it is unstable under all four
        concepts. Symmetrically, $s_3$ is unstable for B.
\end{itemize}
This is the standard result: $s_4$ is the unique Nash equilibrium, and
$s_1$ and $s_4$ are the GMR, SMR and SEQ equilibria. Note that B's arrow
$s_1\xrightarrow{B}s_2$ satisfies $s_1\succ_A s_2$, so \eqref{eq:allweak}
fails for A and, by \cref{thm:nofunctor}, no preference functor for A
exists on this model. Consistently with \cref{cor:collapse}, $s_1$ is GMR
but not Nash stable for A.
\end{example}

\begin{proposition}[Common weak maximum]\label{prop:max}
If $s\in S$ satisfies $s\succsim_i t$ for all $t\in S$ and all $i\in N$,
then $s$ is a Nash equilibrium, and hence an equilibrium under GMR, SMR and
SEQ.
\end{proposition}
\begin{proof}
For every $i$, no $t$ satisfies $t\succ_i s$, so $R_i^{+}(s)=\emptyset$. The
rest follows from \cref{prop:implications}.
\end{proof}

\section{A first step: induced embeddings between graph models}
\label{sec:emb}

Category theory is most useful when it compares structures. We record a first
result on comparing graph models with a common set $N$ of DMs.

\begin{definition}[Induced embedding]\label{def:emb}
Let $G=(N,S,(A_i),(\succsim_i))$ and $G'=(N,S',(A'_i),(\succsim'_i))$ be graph
models. An \emph{induced embedding} $\varphi\colon G\to G'$ is an injective map
$\varphi\colon S\to S'$ such that, for all $s,t\in S$ and $i\in N$,
\[
  (s,t)\in A_i \iff (\varphi(s),\varphi(t))\in A'_i
  \qquad\text{and}\qquad
  s\succsim_i t \iff \varphi(s)\succsim'_i\varphi(t).
\]
\end{definition}

The term ``induced'' refers to the move and preference relations on the
image of $\varphi$, in the sense of induced subgraphs. It does not assert that
the functor between path categories induced by $\varphi$ is full. Indeed,
$\varphi$ sends each arrow $a_{i,s,t}$ of $Q_G$ to the arrow
$a_{i,\varphi(s),\varphi(t)}$ of $Q_{G'}$, and by the universal property of
the free category this defines a functor
$\Path(\varphi)\colon\Path(Q_G)\to\Path(Q_{G'})$. This functor is faithful,
since $\varphi$ is injective on vertices and on arrows, but in general it is
not full: if $G'$ has a state $u\notin\varphi(S)$ and arrows
$\varphi(s)\to u\to\varphi(t)$, the resulting path is a morphism
$\varphi(s)\to\varphi(t)$ in $\Path(Q_{G'})$ that need not be the image of
any path in $Q_G$.

Identities are induced embeddings and composites of induced embeddings are
induced embeddings, so graph models with a fixed set of DMs and induced
embeddings form a category, which we denote $\Emb_N$. For any $F\subseteq S$ such that the
restricted state set $S\setminus F$ contains at least two states, the
inclusion of the restricted model obtained by deleting the states in $F$ (and
all moves incident to them) into $G$ is an induced embedding. So is the inclusion of a
model into one obtained by adding new states, when moves and preferences among
the old states are left unchanged.

\begin{proposition}[Nash stability is reflected]\label{prop:nashreflect}
Let $\varphi\colon G\to G'$ be an induced embedding, $s\in S$ and $i\in N$. If
$\varphi(s)$ is Nash stable for $i$ in $G'$, then $s$ is Nash stable for $i$
in $G$.
\end{proposition}
\begin{proof}
Suppose $t\in R_i^{+}(s)$ in $G$. Then $(s,t)\in A_i$ and $t\neq s$, so
$(\varphi(s),\varphi(t))\in A'_i$ and $\varphi(t)\neq\varphi(s)$ by
injectivity. From $t\succsim_i s$ and not $s\succsim_i t$ we obtain
$\varphi(t)\succsim'_i\varphi(s)$ and not $\varphi(s)\succsim'_i\varphi(t)$.
Hence $\varphi(t)\in R_i^{+}(\varphi(s))$ in $G'$, and $\varphi(s)$ is not
Nash stable for $i$.
\end{proof}

The converse fails, since $G'$ may contain improvements through states outside
$\varphi(S)$. For GMR, SMR and SEQ, neither direction holds.

\begin{example}[GMR, SMR and SEQ are neither preserved nor reflected]
\label{ex:gmr}
Let $N=\{1,2\}$.

\emph{Stable in $G'$, unstable in $G$.} Let $S=\{s,t\}$ with $A_1=\{(s,t)\}$
and $A_2=\emptyset$, and let $S'=\{s,t,u\}$ with $A'_1=\{(s,t)\}$ and
$A'_2=\{(t,u)\}$. Let $t\succ_1 s\succ_1 u$, and let DM~2 satisfy
$u\succ_2 t$, $u\succ_2 s$, $s\sim_2 t$. The inclusion $G\hookrightarrow G'$ is
an induced embedding. In $G$, DM~1's improvement from $s$ to $t$ cannot be
answered, so $s$ is not GMR for DM~1, and hence, by
\cref{prop:implications}, neither SMR nor SEQ. In $G'$, DM~2's move from $t$
to $u$ is an improvement for DM~2 and $s\succ_1 u$; moreover DM~1 has no move
from $u$. Hence $s$ is GMR, SMR and SEQ for DM~1.

\emph{Stable in $G$, unstable in $G'$.} Let $S=\{s,t,u\}$ with
$A_1=\{(s,t)\}$, $A_2=\{(t,u)\}$, and preferences as above. For the same
reasons as in $G'$ above, $s$ is GMR, SMR and SEQ for DM~1 in $G$. Let $S'=S\cup\{w\}$ with $A'_1=\{(s,t),(s,w)\}$ and
$A'_2=A_2$, and extend the preferences so that $w\succ_1 s$ (DM~2's
preferences involving $w$ are arbitrary). The inclusion is an induced
embedding, and in $G'$ the improvement from $s$ to $w$ cannot be answered by
DM~2. Hence $s$ is not GMR for DM~1 in $G'$, and therefore neither SMR nor
SEQ.
\end{example}

\Cref{prop:nashreflect} and \cref{ex:gmr} show a sharp contrast: Nash
stability is reflected by induced embeddings, although it need not be
preserved, whereas GMR, SMR and SEQ, the concepts that involve responses by
other DMs, are neither preserved nor reflected. A systematic study of which
properties of graph models are preserved by which classes of morphisms is left
for future work.

\section{Corrections to version 1}\label{sec:corrections}

\Cref{tab:corrections} lists the definitions and claims of version~1 that
are incorrect or unsupported, together with their replacements in this
version.

\begin{table}[H]
\centering
\footnotesize
\renewcommand{\arraystretch}{1.08}
\begin{tabularx}{\textwidth}{@{}>{\raggedright\arraybackslash}p{3.1cm}>{\raggedright\arraybackslash}X>{\raggedright\arraybackslash}X@{}}
\toprule
Item in version 1 & Problem & Replacement in this version \\
\midrule
C-GMCR tuple and C-Conflict: states and state transitions form a category &
The proposed tuple was internally inconsistent (described as ``five tuples''
while listing six components), and the categorical component was not well
formed: one-step moves are not closed under composition, and moves of different
DMs between the same states were not distinguished. &
The extra C-GMCR tuple is abandoned. Starting from the ordinary graph model
$G=(N,S,(A_i),(\succsim_i))$, we construct the labelled quiver $Q_G$ and its
free category $\Path(Q_G)$ (\cref{def:quiver,rem:notcategory}). \\
Preference functor $P_i$ &
The object part requires transitive preferences, which GMCR does not assume;
the arrow part exists only if every move is weakly improving for $i$. &
\Cref{prop:transitive,thm:nofunctor}; preferences act as a selection of
$Q_G^{+}$ (\cref{rem:selection}). \\
Proposition 1 &
False as stated: with a reversible option, a move $s_1\to s_2$ with
$s_2\succ_i s_1$ and a move $s_2\to s_1$ coexist (\cref{ex:pd}). The proof
relied on the preference functor. &
The intended content, that a move and its reverse cannot both be improvements
for the same DM, follows from asymmetry of $\succ_i$ alone. \\
Proposition 2 &
A restatement of the definition of equilibrium. &
Omitted. \\
Generic ``C-GMCR equilibrium'' &
Defined equilibrium by the absence of a unilateral improvement, which
characterises Nash stability rather than equilibrium under GMR, SMR or SEQ. &
Replaced by concept-specific equilibrium: a state is an equilibrium under a
stability concept exactly when it is stable for every DM under that concept
(\cref{def:stab}). \\
Proposition 3 &
Correct in substance. &
Restated for complete, possibly intransitive preferences and extended to all
four concepts (\cref{prop:max}). \\
Definitions of C-GMCR-GMR, -SMR and -SEQ &
Do not coincide with the standard concepts (e.g.\ C-GMR quantified over
sequences of weakly improving moves by all DMs; C-SEQ required an improving
sequence leading \emph{to} the state). &
Standard definitions (\cref{def:stab}) and their path form
(\cref{prop:pathstab}). \\
C-Preference on morphisms &
Introduced but never used; moreover, the stated axioms were not mutually
consistent. In particular, the condition written for strict preference,
$f\succ_i^c g\Leftrightarrow\neg(g\succ_i^c f)$, is stronger than asymmetry
and conflicts with the simultaneous allowance of indifference. &
Omitted. Preferences remain binary relations on states; strict preference and
indifference are derived from $\succsim_i$ (\cref{def:gm}). \\
Transitive preferences; ``no loops'' together with $(s,s)\in A$ &
Transitivity is stronger than GMCR requires; the two statements on loops were
inconsistent. &
Complete, possibly intransitive preferences; pairs $(s,s)$ allowed, reachable
lists defined over proper moves (\cref{def:gm}). \\
Claim that composition of morphisms reveals higher-order strategies not
visible in GMCR &
GMCR already quantifies over legal move sequences through coalition
reachability. &
Withdrawn. The categorical content of move sequences is made explicit in
\cref{prop:legal}. \\
Examples (PD, Elmira, trade) &
Move sets were incompletely specified (e.g.\ reverse moves omitted in the PD);
the trade example mixed rules with moves. &
Fully specified PD (\cref{ex:pd}); other examples omitted. \\
Concluding claims on verification, explainability and AI &
Not supported by results in the paper. &
Withdrawn. \\
\bottomrule
\end{tabularx}
\caption{Corrections to version~1.}
\label{tab:corrections}
\end{table}

\section{Conclusion}

The correspondence between GMCR and category theory is best stated at two
levels. Within a single graph model, the move structure is captured by the
free category on the labelled quiver together with the controller-labelling
functor, which together determine coalition reachability. Preferences are not
functorial on this category except in the degenerate case in which no move
ever harms the DM concerned, and in that case the stability concepts beyond
Nash stability collapse for that DM. The distinctive content of GMCR solution
concepts lies in comparing the endpoints of move sequences with their starting
states, which composition alone does not express. Between graph models,
induced embeddings give a first category in which to compare models: Nash
stability is reflected along them, whereas GMR, SMR and SEQ are not preserved
in either direction. Understanding which properties survive which kinds of
model-to-model maps is the natural next question.



\begin{thebibliography}{99}

\bibitem{Awodey2010}
S.~Awodey, \emph{Category Theory}, 2nd ed.
Oxford University Press, 2010.

\bibitem{EM1945}
S.~Eilenberg and S.~Mac Lane,
``General theory of natural equivalences,''
\emph{Trans. Amer. Math. Soc.}, vol.~58, no.~2, pp.~231--294, 1945.

\bibitem{FHK1993}
L.~Fang, K.~W. Hipel, and D.~M. Kilgour,
\emph{Interactive Decision Making: The Graph Model for Conflict Resolution}.
New York: Wiley, 1993.

\bibitem{FraserHipel1979}
N.~M. Fraser and K.~W. Hipel,
``Solving complex conflicts,''
\emph{IEEE Trans. Syst. Man Cybern.}, vol.~9, no.~12, pp.~805--816, 1979.

\bibitem{FraserHipel1984}
N.~M. Fraser and K.~W. Hipel,
\emph{Conflict Analysis: Models and Resolutions}.
New York: North-Holland, 1984.

\bibitem{Howard1971}
N.~Howard,
\emph{Paradoxes of Rationality: Theory of Metagames and Political Behavior}.
Cambridge, MA: MIT Press, 1971.

\bibitem{KHF1987}
D.~M. Kilgour, K.~W. Hipel, and L.~Fang,
``The graph model for conflicts,''
\emph{Automatica}, vol.~23, no.~1, pp.~41--55, 1987.

\bibitem{MacLane1998}
S.~Mac Lane,
\emph{Categories for the Working Mathematician}, 2nd ed.
New York: Springer, 1998.

\bibitem{Nash1950}
J.~F. Nash,
``Equilibrium points in $n$-person games,''
\emph{Proc. Natl. Acad. Sci. USA}, vol.~36, no.~1, pp.~48--49, 1950.

\bibitem{Nash1951}
J.~Nash,
``Non-cooperative games,''
\emph{Ann. Math.}, vol.~54, no.~2, pp.~286--295, 1951.

\bibitem{Riehl2016}
E.~Riehl,
\emph{Category Theory in Context}.
Mineola, NY: Dover, 2016.

\end{thebibliography}
\end{document}